\documentclass[letterpaper,11pt]{article}
\usepackage[singlespacing]{setspace}
\usepackage[margin=1in]{geometry}
\usepackage[T1]{fontenc}
\usepackage[utf8]{inputenc}

\usepackage{fullpage}
\usepackage{graphicx}
\usepackage{amsmath,amsthm,amssymb}
\usepackage{algorithm}
\usepackage{algorithmic}
\usepackage{wrapfig}
\usepackage{xcolor}
\usepackage{natbib}

\newtheorem{theorem}{Theorem}
\newtheorem{lemma}[theorem]{Lemma}

\def\squareforqed{\hbox{\rule{2.5mm}{2.5mm}}}

\def\QED{\ifmmode\squareforqed 
  \else{\nobreak\hfil   
    \penalty50                 
    \hskip1em                  
    \null                      
    \nobreak                   
    \hfil                      
    \squareforqed              
    \parfillskip=0pt           
    \finalhyphendemerits=0     
    \endgraf}                  
  \fi}

\def\blksquare{\rule{2mm}{2mm}}
\def\qedsymbol{\blksquare}
\newcommand{\bgg}[1]{\medskip\noindent{\bf #1}}
\newcommand{\ed}{{\hfill\qedsymbol}\medskip}

\newcommand{\field}[1]{\mathbb{#1}}
\newcommand{\R}{\field{R}}

\newcommand{\inner}[1]{ \left\langle {#1} \right\rangle }

\newcommand{\Reg}{\text{\rm Reg}}

\newlength{\minipagewidth}
\begin{document}

\title{A Short Note on a Variant of the Squint Algorithm}
\author{Haipeng Luo\\
    University of Southern California
    }
\date{}
\maketitle

\abstract{
This short note describes a simple variant of the Squint algorithm of~\citet{koolen2015second} for the classic expert problem.
Via an equally simple modification of their proof, we prove that this variant ensures a regret bound that resembles the one shown in a recent work by~\citet{freund2026second} for a variant of the NormalHedge algorithm~\citep{chaudhuri2009parameter}.
}

\section{Expert Problem and Squint}
In the classic expert problem, a learner interacts with an adversary for $T$ rounds.
In each round $t=1, \ldots, T$, 
the learner decides a distribution $p_t \in \Delta_N$ over $N$ experts,
and the adversary decides the loss $\ell_t \in [0,1]$.
The learner then suffers loss $\inner{p_t, \ell_t}$ and observes $\ell_t$.
For any $\epsilon \in [1/N,1)$, the $\epsilon$-quantile regret of the learner is defined as
\[
\Reg_\epsilon = \sum_{t=1}^T \inner{p_t, \ell_t} - \sum_{t=1}^T \ell_{t, i_\epsilon}
\]
where $i_\epsilon$ is the $\lfloor \epsilon N \rfloor$-th best expert (according to their cumulative loss over $T$ rounds).
When $\epsilon=1/N$, $\Reg_\epsilon$ reduces to the standard external regret that compares the total loss of the learner to that of the best expert in hindsight.

Define the instantaneous regret vector $r_t$ as $\inner{p_t, \ell_t}\mathbf{1} - \ell_t$ (where $\mathbf{1}\in\R^{N}$ is the all-one vector) and the cumulative regret vector as $R_t = \sum_{s=1}^t r_t$.
Further define the Squint potential as
\[ 
\Phi(R, V) = \int_0^{1/2} \frac{e^{\eta R - \eta^2 V}-1}{\eta} d\eta 
\]
for $R \in \R$ and $V \in \R^+$, which satisfies the following property.

\begin{lemma}\label{lem:squint}
For any $x \in [-1,1], R\in \R$ and $V \in \R^+$, we have 
\[ \Phi(R+x, V+x^2) \leq \Phi(R, V) + \frac{\partial\Phi}{\partial R}(R,V) \cdot x. \]
\end{lemma}
\begin{proof}
By definition, one has
\begin{align*} 
\Phi(R+x, V+x^2) &= \int_0^{1/2} \frac{e^{\eta R - \eta^2 V}e^{\eta x - \eta^2x^2}-1}{\eta} d\eta \\
&\leq \int_0^{1/2} \frac{e^{\eta R - \eta^2 V}(1+\eta x)-1}{\eta} d\eta = \Phi(R, V) + \frac{\partial\Phi}{\partial R}(R,V) \cdot x, 
\end{align*}
where we use the inequality $e^{y-y^2} \leq 1+y$ for any $y\geq -\tfrac{1}{2}$.
\end{proof}

\paragraph{Squint Algorithm and Guarantee}
The original Squint algorithm~\citep{koolen2015second} predicts $p_t \in \Delta_N$ at time $t$ such that
\begin{equation}
p_{t,i} \propto \frac{\partial \Phi}{\partial R}(R_{t-1,i}, V_{t-1,i}), 
\tag{original Squint}
\end{equation}
where $V_{t,i} = \sum_{s=1}^t v_{t,i}$ and $v_{t,i} = r_{t,i}^2$.
The following lemma shows that the sum of the potential of the algorithm never increases, which is the key for the analysis.

\begin{lemma}[Lemma 1 of~\citet{koolen2015second}]\label{lem:Squint_potential_lemma}
The Squint algorithm ensures 
\[
\sum_{i=1}^N \phi(R_{T,i}, V_{T,i})
\leq \sum_{i=1}^N \phi(R_{T-1,i}, V_{T-1,i})
\leq \cdots \leq \sum_{i=1}^N \phi(R_{0,i}, V_{0,i}) = 0
\]
\end{lemma}
\begin{proof}
For any $t$, we have
\begin{align*}
\sum_{i=1}^N \Phi(R_{t,i}, V_{t,i}) 
&= \sum_{i=1}^N \Phi(R_{t-1,i}+r_{t,i}, V_{t-1,i}+r_{t,i}^2)  \\
&\leq \sum_{i=1}^N \Phi(R_{t-1,i}, V_{t-1,i}) + \frac{\partial \Phi}{\partial R}(R_{t-1,i}, V_{t-1,i}) \cdot r_{t,i} \tag{Lemma~\ref{lem:squint}} \\
&= \sum_{i=1}^N \Phi(R_{t-1,i}, V_{t-1,i}),
\end{align*}    
where the last equality is because $\sum_{i=1}^N p_{t,i} r_{t,i} = 0$ and $p_{t,i} \propto \frac{\partial \Phi}{\partial R}(R_{t-1,i}, V_{t-1,i})$.
\end{proof}

Solely based on the fact $\sum_{i=1}^N \phi(R_{T,i}, V_{T,i}) \leq 0$,
\citet[Theorem~4]{koolen2015second} then prove the following regret bound for Squint, which holds simultaneously for all $\epsilon$:
\[
\Reg_\epsilon \leq \sqrt{2 V_{T, i_\epsilon}}
\left(
1 + \sqrt{ 2 \ln\!\left( \frac{\frac{1}{2} + \ln(T+1)}{\epsilon} \right) }
\right)
+ 5 \ln\!\left(
1 + \frac{1 + 2 \ln(T+1)}{\epsilon}
\right).
\]

\section{A Variant of Squint}
Now, consider the following variant of Squint:
at time $t$,  predicts $p_t \in \Delta_N$ such that
\begin{equation} 
p_{t,i} \propto \frac{\partial \Phi}{\partial R}(R_{t-1,i}, {\color{red}V_{t-1}}),  \tag{Squint varaint}
\end{equation}
where $V_{t} = \sum_{s=1}^t v_{t}$, {\color{red}$v_{t} = \sum_{i=1}^N q_{t,i} r_{t,i}^2$}, and
\[
q_{t,i} \propto -\frac{\partial \Phi}{\partial V}(R_{t,i}, V_{t}) = \frac{\partial^2 \Phi}{\partial R^2}(R_{t,i}, V_{t}) = \int_0^{1/2} \eta e^{\eta R_{t,i} - \eta^2 V_t} d\eta.
\]
Although the definition of $v_t$ is recursive (since it depends on $q_t$, which itself depends on $v_t$), one can still find $v_t$ efficiently via a simple line search.
Specifically,
note that $v_t$ is the root of the function 
\[
f(v) = \sum_{i=1}^N \frac{\partial^2 \Phi}{\partial R^2}(R_{t,i}, V_{t-1}+v)\left(v - r_{t,i}^2\right).
\]
The facts that $f(v)$ is continuous, $f(0)\leq 0$, and $f(1) \geq 0$ imply that a root of $f$ must exist because $0$ and $1$ and that it can be found via binary search.

\paragraph{Analysis}
Via a simple modification of the proof of Lemma~\ref{lem:Squint_potential_lemma}, 
one can show that the sum of the potential of this Squint variant also never increases.

\begin{lemma}
The Squint variant ensures 
\[
\sum_{i=1}^N \phi(R_{T,i}, V_{T})
\leq \sum_{i=1}^N \phi(R_{T-1,i}, V_{T-1})
\leq \cdots \leq \sum_{i=1}^N \phi(R_{0,i}, V_{0}) = 0
\]
\end{lemma}
\begin{proof}
For any $t$, we have
\begin{align*}
\sum_{i=1}^N \Phi(R_{t,i}, V_t) 
&\leq \sum_{i=1}^N \Phi(R_{t,i}, V_{t-1}+r_{t,i}^2) + \frac{\partial \Phi}{\partial V}(R_{t,i}, V_t)(v_t - r_{t,i}^2) \tag{convexity of $\Phi$ in $V$} \\
&= \sum_{i=1}^N \Phi(R_{t,i}, V_{t-1}+r_{t,i}^2)  \tag{definition of $v_t$} \\
&= \sum_{i=1}^N \Phi(R_{t-1,i}+r_{t,i}, V_{t-1}+r_{t,i}^2)  \\
&\leq \sum_{i=1}^N \Phi(R_{t-1,i}, V_{t-1}) + \frac{\partial \Phi}{\partial R}(R_{t-1,i}, V_{t-1}) \cdot r_{t,i} \tag{Lemma~\ref{lem:squint}} \\
&= \sum_{i=1}^N \Phi(R_{t-1,i}, V_{t-1}),
\end{align*}
where the last equality is again because $\sum_{i=1}^N p_{t,i} r_{t,i} = 0$ and $p_{t,i} \propto \frac{\partial \Phi}{\partial R}(R_{t-1,i}, V_{t-1})$.
\end{proof}

Therefore, repeating the exact same arguments of~\citet[Theorem~4]{koolen2015second}, we obtain the following theorem.
\begin{theorem}\label{thm:main}
The $\epsilon$-quantile regret of the Squint variant satisfies:
\[
\Reg_\epsilon \leq \sqrt{2 V_{T}}
\left(
1 + \sqrt{ 2 \ln\!\left( \frac{\frac{1}{2} + \ln(T+1)}{\epsilon} \right) }
\right)
+ 5 \ln\!\left(
1 + \frac{1 + 2 \ln(T+1)}{\epsilon}
\right).
\]
simultaneously for all $\epsilon$.
\end{theorem}
The only difference of this bound compared to that of the original Squint is the replacement of $V_{T, i_\epsilon}$ with $V_T$.
The two bounds are incomparable in general, but the one for the new Squint variant resembles the bound from a recent work of~\citet{freund2026second} for a different algorithm.
Specifically, they analyze a variant of the NormalHedge algorithm of~\citet{chaudhuri2009parameter} and prove a similar bound with a different definition of the potential function.

As a final remark, we note that using the same idea as~\citet[Theorem~1]{luo2015achieving}, one can scale the update rule by any prior distribution $q\in\Delta_N$ (that is, $p_{t,i} \propto q_i\frac{\partial \Phi}{\partial R}(R_{t-1,i}, {V_{t-1}})$),
and convert the adaptive quantile bound of Theorem~\ref{thm:main} into a regret bound against any distribution $u\in\Delta_N$, replacing the dependency of $\ln(1/\epsilon)$ with $\text{KL}(u, q)$.

\paragraph{Acknowledgment}
The author thanks Tim Van Erven and Wouter Koolen for discussion related to this variant during his visit to Centrum Wiskunde \& Informatica in 2016.
He also thanks Yoav Freund and Yu-Xiang Wang for discussion related to their NormalHedge variant.

\bibliographystyle{plainnat} 
\bibliography{ref}

\end{document}